\documentclass[journal, twocolumn]{IEEEtran}

\usepackage{graphicx}

\usepackage{amsmath}
\usepackage{amssymb}
\usepackage{amsfonts}
\usepackage{amsthm}
\usepackage[switch]{lineno} 

\usepackage[compact]{titlesec}
\newtheorem{theorem}{Theorem}
\newtheorem{problem}{Problem}

\usepackage{etoolbox}
\makeatletter
\patchcmd{\@makecaption}
  {\scshape}
  {}
  {}
  {}
\makeatother

\usepackage{array}

\usepackage[table]{xcolor}
\definecolor{my_gray}{HTML}{DBDBDB}
\usepackage{booktabs, makecell, multirow}
\usepackage{makecell}

\newcommand{\bsC}{\mathbf{C}}

\newcommand{\bsD}{\mathbf{D}}
\newcommand{\bsJ}{\mathbf{J}}
\newcommand{\bsV}{\mathbf{V}}
\newcommand{\bsy}{\mathbf{y}}

\newcommand{\bsg}{\mathbf{g}}
\newcommand{\bsx}{\mathbf{x}}
\newcommand{\bsv}{\mathbf{v}}
\newcommand{\bsz}{\mathbf{z}}

\newcommand{\bsW}{\mathbf{W}}
\newcommand{\bsm}{\mathbf{m}}

\newcommand{\bsmu}{\boldsymbol{\mu}}
\newcommand{\bsd}{\mathbf{d}}
\newcommand{\bstheta}{\boldsymbol{\theta}}
\newcommand{\bsTheta}{\boldsymbol{\Theta}}

\newcommand{\bsSigma}{\boldsymbol{\Sigma}}

\DeclareMathOperator*{\argmin}{argmin}

\begin{document}

\title{Latent Score-Based Bayesian Cram{\'e}r--Rao Bound Estimation for High-Dimensional Imaging Systems}

\author{Evan~Scope~Crafts, Thomas Wynn, Seonyeong Park, Mark Anastasio, and Umberto~Villa%

\thanks{This work was supported by the National Institutes of Health under grants EB031585 and EB034261. The authors acknowledge the Texas Advanced Computing Center (TACC) at The University of Texas at Austin for providing computational resources on the Vista system.
}
\thanks{E. Scope Crafts is with the Oden Institute for Computational Engineering and Sciences, the University of Texas at Austin, Austin, TX, 78712.}%
\thanks{T. Wynn is with the Graduate School of Biomedical Sciences at UT Southwestern Medical Center.}
\thanks{S. Park and M. Anastasio are with the Mallinckrodt Institute of Radiology at Washington University in St. Louis.}
\thanks{U. Villa is with the Department of Biomedical Engineering and the Oden Institute for Computational Engineering and Sciences, the University of Texas at Austin, Austin, TX, 78712.}}

\maketitle

\begin{abstract}
We propose a data-driven framework for estimating the Bayesian Cram\'er--Rao bound (CRB) in high-dimensional imaging systems with complex, analytically intractable priors. Direct CRB computation is challenging in this setting due to the need to model the prior score and to form and invert the Bayesian Fisher information matrix in very high dimensions. To address these issues, we first reformulate the inverse problem in the latent space of a pre-trained variational autoencoder, thereby dramatically reducing the dimensionality of the bound estimation problem while preserving the spatial structure of the images. The Bayesian CRB is formed in this latent space and mapped back to the native parameter space using a change-of-variables formula. Second, to learn the latent prior score, we introduce a new sliced score matching objective defined in a Bochner space endowed with an $H^1(\Omega)$ Sobolev norm in space. This ``Bochner-space sliced score matching'' objective is consistent with standard sliced score matching, but penalizes errors in the spatial gradients of the score, suppressing high-frequency artifacts that otherwise contaminate the resulting CRB estimates. We validate the approach on a stylized quantitative photoacoustic computed tomography (qPACT) breast imaging problem with over one million unknown parameters, using a foundation-model autoencoder derived from Stable Diffusion. The proposed method yields stable, artifact-free Bayesian CRB estimates that reflect the highly non-Gaussian structure of the learned prior and reveal the substantial impact of the prior on the relative performance of competing qPACT design schemes.
\end{abstract}
\section{Introduction}

The Bayesian Cram\'{e}r--Rao bound (Bayesian CRB) \cite{bell_2013_detection, trees_2007_bayesian} provides a lower bound on the mean squared error of any estimator for a Bayesian inverse problem. This property makes it a principled benchmark for estimator performance~\cite{Tichavsky1998} and enables its use as a criterion for imaging system design~\cite{akccakaya2015selection, crafts2024bayesian}. 
Crucially, the Bayesian CRB can be computed without 
solving the underlying inverse problem and, thus, does not depend on the choice of estimator. 
Consequently, Bayesian CRB-based design methods are computationally more efficient than approaches 
that require repeatedly solving the inverse problem to evaluate candidate designs. 

There are two key challenges in 
forming the Bayesian CRB for imaging systems. First, the Bayesian CRB depends on the scores, defined as 
the gradients of the log-densities, 
of both the prior distribution and the likelihood function. In imaging systems, high-fidelity models of the likelihood function can often be derived from 
domain-specific knowledge of 
the physics underlying 
the imaging modality. In contrast, 
the prior distribution is typically unknown and difficult to model accurately. Second, the Bayesian CRB is canonically defined as the inverse of the Bayesian Fisher information matrix, a square matrix whose 
number of rows and columns equals the dimension of the parameter space 
of the underlying inverse problem. As imaging systems often involve 
unknown parameters 
in high-dimensional, or even infinite-dimensional, parameter spaces, computing the inverse of the Bayesian Fisher information matrix 
can be prohibitively expensive, and storing the matrix in memory may be infeasible. 

\subsection{Bayesian CRB Estimation}

In recent years, there has been significant work to address the first challenge described above by developing data-driven Bayesian CRB estimators for settings where the prior distribution and/or 
the likelihood function are unknown. For example, in \cite{crafts2023}, a Bayesian CRB estimator was developed that uses 
score matching \cite{hyvarinen_2005_estimation}, a parametric 
technique for estimating the score of a data distribution from samples, to model the score of the prior distribution. The estimated prior score and the known likelihood function are then used to construct a Monte Carlo estimator of the Bayesian CRB. In \cite{habi2025learned}, this approach was extended to 
settings where the likelihood is either fully unknown 
or has a known 
forward model but unknown 
measurement noise statistics. In a different vein, Duy et al. \cite{duy_fisher_2023} proposed an approach for directly estimating the entries of the Bayesian Fisher information matrix 
by exploiting a relationship between 
Fisher information and $f$-divergence measures \cite{berisha_empirical_2015}. 

The above approaches do not readily 
scale to high-dimensional settings. Specifically, the score matching objective used to learn the prior score in \cite{crafts2023} and \cite{habi2025learned} requires evaluating the trace of the score model Jacobian, which has been shown to be computationally expensive for 
high-dimensional neural-network-based score models \cite{martens2010deep}. As a consequence, these approaches have 
been validated only on example problems with parameter dimensions of twenty or fewer. Further, these approaches do not address the difficulties associated with forming and inverting the Bayesian Fisher information matrix 
that arise in very high-dimensional problem settings. The approach proposed by 
Duy et al.~\cite{duy_fisher_2023} is even more limited in its capacity to address high-dimensional problems, as it requires training $D (D + 1)/2$ separate neural networks, where $D$ denotes 
the parameter dimension. 

Another approach for estimating the Bayesian CRB 
was recently introduced by the authors in \cite{crafts2024bayesian}. This approach, unlike the score-matching and neural network approaches described above, was demonstrated to scale to high-dimensional problems. 
However, a key ingredient in achieving this scalability is the use of a log-Gaussian prior with Mat{\'e}rn covariance \cite{lindgren2011explicit} to capture the prior structure. Although this choice of 
prior makes 
the Bayesian Fisher information matrix and its inverse tractable to compute in 
high-dimensional Bayesian inverse problems, it cannot capture 
the intricate structure of image distributions. 

\subsection{Our Contributions}

In this work, we introduce a novel approach to 
Bayesian CRB estimation that 
addresses two 
challenges commonly encountered in imaging system design: 
high-dimensional parameter spaces and complex, analytically 
unknown prior distributions. 
This approach relies on two key innovations. First, to reduce the computational burden of Bayesian CRB estimation in high-dimensional inverse problems, 
we reformulate the inverse problem in the latent space of an autoencoder. This technique has been widely explored for 
solving inverse problems \cite{bora2017compressed, asim2020invertible}, but, to the best of our 
knowledge, it has not been employed for experimental design. We exploit this 
latent-space formulation by computing 
the Bayesian CRB in the latent space and transforming it to 
the original parameter space using 
the change-of-variables formula \cite{trees_2007_bayesian}. 

The latent-space formulation alleviates the computational difficulties associated with forming and inverting the Bayesian Fisher information matrix 
in high-dimensional settings and 
makes estimation of the prior score more tractable. However, the dimension of the autoencoder latent space can remain 
relatively large, exceeding 
$1000$ in many applications. 
The computational cost of evaluating 
the score matching (SM) objective in such spaces 
can be mitigated through the use of sliced score matching (SSM) \cite{song_2020_sliced}, an SM variant that uses Hutchinson's trace estimator \cite{hutchinson1989stochastic} to approximate 
the Jacobian trace term in the original SM objective. However, in our numerical experiments, we found that directly applying SSM in the latent space can introduce 
artifacts in the learned score that degrade 
the resulting Bayesian CRB estimates. 

Our second contribution is therefore a novel variant of the SSM objective, referred to here as \emph{sliced score matching in Bochner spaces}, designed to mitigate these artifacts. 
The proposed SM objective is motivated by the observation that the latent space of autoencoders generally retains much of the spatial structure of the original parameter space. 
For parameter spaces with spatial structure, both SM and SSM can be interpreted 
as spatial discretizations of a loss function in a \emph{Bochner space}. Bochner spaces, which generalize 
scalar-valued $L^p$ function spaces to functions taking values in a Banach space 
\cite{hytonen2016analysis}, require the choice of 
a norm for 
the codomain space. The 
standard SM 
and SSM objectives correspond to Bochner spaces with an $L^2$ codomain norm. The proposed SSM objective essentially replaces the $L^2$ norm in the above formulations with an $H^1$ Sobolev 
norm, which penalizes errors in the spatial gradient of the function \cite{czarnecki2017sobolev}. The corresponding discretized objective, like the original SSM objective, yields consistent estimators of the score. However, compared to the original SSM loss, the $H^1$-weighted variant strongly reduces high-frequency components 
of the score error that can appear as artifacts in the resulting bound estimates. 

We evaluate 
the proposed estimator on a stylized version of the inverse problem encountered in quantitative photoacoustic computed tomography (qPACT) \cite{cox2009challenges, tarvainen2024quantitative}, an emerging biomedical imaging modality. In particular, we consider the computation of the Bayesian CRB for a qPACT breast imaging system with a parameter space of dimension $1024 \times 1024$ and training data generated from 
the stochastic numerical breast phantoms introduced in \cite{park2023stochastic, park2026virtual}. Here, 
the inverse problem involves estimating 
the optical absorption coefficient from 
single-wavelength qPACT data, assuming 
a known optical scattering coefficient. In this setting, we show that the proposed approach produces 
bound estimates that accurately capture the structure in both the likelihood and the prior distribution. 
We further show that the proposed SSM objective mitigates 
artifacts that appear in the bound estimate 
when 
the original SSM objective is used. 

\subsection{Organization}

The remainder of this paper 
is organized as follows. Section \ref{sec:background} provides relevant background on SM, SSM, 
and the Bayesian CRB. Section \ref{sec:methods} formally introduces the Bayesian CRB estimation problem, together with the proposed estimator. Section \ref{sec:num_stud} provides details of 
the numerical experiments we conducted, while Section \ref{sec:results} contains the results. Concluding remarks are given in Section \ref{sec:discuss_conclude}.

\section{Background}
\label{sec:background}
In this section, we provide relevant background regarding SM 
and the Bayesian CRB.

\subsection{Score Matching}

Score matching is a statistical technique for estimating the score of a given probability distribution from a set of i.i.d. samples $\{\bsm_i \in \mathbb{R}^D\}_{i=1}^N$, where $\bsm_i \sim p(\bsm)$. It estimates the score by 
minimizing the Fisher divergence $L(\bstheta)$ between 
a parametric model of the score (the \textit{score model}) and the ground-truth score: 
\begin{equation}
\label{eq:fisher}
    L(\bstheta) \triangleq \frac{1}{2} \mathbb{E}_{\bsm} \left [ \left \| s_{\bstheta}(\bsm) - \nabla_{\bsm} \log p(\bsm) \right \|_2^2 \right].
\end{equation}
In the above expression, $s_{\bstheta}: 
\mathbb{R}^D \to \mathbb{R}^D$ denotes 
the score model, parameterized by 
$\bstheta \in \bsTheta \subset \mathbb{R}^P$. Note that the score of the ground-truth distribution of interest, $\nabla_{\bsm} \log p(\bsm)$, is unknown. However,
Hyv{\"a}rinen \cite{hyvarinen_2005_estimation} showed that, under mild regularity conditions, the identity $L(\bstheta) = J(\bstheta) + C$ holds, where $C$ is a constant independent of the score model and 
\begin{equation}
    \label{eq:sm}
    J(\bstheta) \triangleq \frac{1}{2} \mathbb{E}_{\bsm} \left [ \left \| s_{\bstheta}(\bsm) \right \|_2^2 + 2 \, \mathrm{tr}\left( \nabla_{\bsm} s_{\bstheta}(\bsm) \right) \right ].
\end{equation}
Here, $\mathrm{tr}(\cdot)$ denotes the trace of the given matrix, and 
$\nabla_{\bsm} s_{\bstheta}(\bsm) \in \mathbb{R}^{D \times D}$ is the Jacobian of the score model. Note that the terms inside 
the expectation in \eqref{eq:sm} depend only on the score model and can therefore be evaluated without knowledge of the ground-truth score. Consequently, the expectation can be approximated using Monte Carlo methods and the available 
i.i.d.\ samples. 

Although 
the above SM formulation 
yields a consistent estimator of the score, its Jacobian term can be computationally expensive to evaluate. 
To reduce this cost, 
Song et al.~introduced SSM \cite{song_2020_sliced}, which applies 
Hutchinson's trace estimator \cite{hutchinson1989stochastic} to the Jacobian term. 
The resulting objective can be written as follows:
\begin{equation}
    \label{eq:ssm}
        J_{\mathrm{sl}}(\bstheta) \triangleq \mathbb{E}_{\bsv} \, \mathbb{E}_{\bsm}  \left [ \frac{1}{2} \left \| s_{\bstheta}(\bsm) \right \|_2^2 + \bsv^T \left [\nabla_{\bsm} s_{\bstheta}(\bsm) \right] \bsv \right],
\end{equation}
where $\bsv$ follows 
a distribution satisfying $\mathbb{E} [ \| \bsv \|_2^2 ] < \infty$ and $\mathbb{E} [ \bsv \bsv^T ] = \mathbf{I}$. 
In this work, we use $\bsv \sim \mathcal{N}(\mathbf{0}, \mathbf{I})$. This objective admits the following unbiased estimator:
\begin{equation}
\label{eq:ssm_empirical}
\begin{split}
    \hat{J}_{\mathrm{sl}}(\bstheta) \triangleq  \frac{1}{N} \sum_{i=1}^N & \left( \frac{1}{2}\left \| s_{\bstheta}(\bsm_i) \right \|_2^2 +\right. \\
    & \left. \frac{1}{M} \sum_{j=1}^M \, \bsv_{ij}^T \left [\nabla_{\bsm} s_{\bstheta}(\bsm_i) \right] \bsv_{ij} \right),
\end{split}
\end{equation}
where 
$M \geq 1$ denotes 
the number of random vectors used to estimate the trace, and the $\bsv_{ij}$ are i.i.d. samples. Note that the projected Jacobian term $\bsv_{ij}^T \left [\nabla_{\bsm} s_{\bstheta}(\bsm_i) \right] \bsv_{ij}$ can be computed efficiently using a single reverse-mode automatic differentiation pass by exploiting the identity 
$ \left [\nabla_{\bsm} s_{\bstheta}(\bsm_i) \right] \bsv_{ij} = \nabla_{\bsm} \left [ s_{\bstheta}(\bsm_i)^T \bsv_{ij} \right]$. Song et al. proved that, under 
Assumptions 1--8 in \cite{song_2020_sliced}, 
the score estimator obtained from \eqref{eq:ssm_empirical} is 
consistent 
for any $M \geq 1$. 

\subsection{The Bayesian CRB}

The Bayesian CRB is an information-theoretic lower bound on the error of \emph{any} estimator under mild regularity conditions. It is widely used both to 
benchmark the performance of estimators and to 
guide system design, with scalar invariants of the bound serving as metrics for 
the analysis and optimization of signal processing systems. For an inverse problem with unknown parameter $\bsm \in \mathbb{R}^D$ and measurement (observation) $\bsy \in \mathbb{R}^K$, the bound can be written as follows \cite{bell_2013_detection}:
\begin{equation}
    \mathbb{E}_{\bsm, \bsy} \left [ (\hat{\bsm}(\bsy) - \bsm)  (\hat{\bsm}(\bsy) - \bsm)^T \right ] \succeq \bsV_{\bsm}, \quad \bsV_{\bsm} \triangleq \bsJ_{\bsm}^{-1}. 
\end{equation}
Here, $\hat{\bsm}: 
\mathbb{R}^K \to \mathbb{R}^D$ denotes 
an estimator, $\bsV_{\bsm} \in \mathbb{R}^{D \times D}$ is the Bayesian CRB for the parameter $\bsm$, and 
$\mathbf{A} \succeq \mathbf{B}$ indicates 
that $\mathbf{A} - \mathbf{B}$ is 
positive semi-definite. 
The matrix $\bsJ_{\bsm} \in \mathbb{R}^{D \times D}$ is the corresponding \emph{Bayesian Fisher information matrix}, defined as 
$$
\bsJ_{\bsm} \triangleq \mathbb{E}_{\bsm, \bsy} \left [ \nabla_{\bsm} \log \pi_{\mathrm{jnt}} (\bsm, \bsy) \, \nabla_{\bsm} \log \pi_{\mathrm{jnt}}(\bsm, \bsy)^T \right],
$$
where $\pi_{\mathrm{jnt}}$ is the joint density of $\bsm$ and $\bsy$, which can be written as
$$
\pi_{\mathrm{jnt}}(\bsm, \bsy) \triangleq \pi_{\mathrm{like}}(\bsy \mid \bsm) \, \pi_{\mathrm{pr}}(\bsm). 
$$
Here, $\pi_\mathrm{like}(\bsy\mid\bsm)$ denotes the likelihood function and $\pi_\mathrm{pr}(\bsm)$ denotes the prior density.

The Bayesian Fisher information matrix admits the following decomposition: 
\begin{equation}
\bsJ_{\bsm } = \bsJ^P_{\bsm} + \bsJ^D_{\bsm}.
\label{eq:Jb_exact}
\end{equation}
Here, the \emph{prior term} $\bsJ^P_{\bsm}$ 
is given by
\begin{equation}
\begin{split}
\bsJ^P_{\bsm} & \triangleq  \mathbb{E}_{\bsm} \left [ \nabla_{\bsm} \log \pi_{\mathrm{pr}}(\bsm) \; \nabla_{\bsm} \log \pi_{\mathrm{pr}}(\bsm)^T \right ]\\ {} &=  - \mathbb{E}_{\bsm} \left [  \nabla_{\bsm} [\nabla_{\bsm} \log \pi_{\mathrm{pr}}(\bsm) ] \right ],
\end{split}
\label{eq:Jp}
\end{equation}
where 
the second equality follows from 
integration by parts (see, e.g., \cite{bell_2013_detection}). 
The \emph{data term} $\bsJ_\bsm^D$ can be written as
\begin{equation}
\begin{split}
\bsJ_\bsm^D & \triangleq \mathbb{E}_{\bsm, \bsy} \left [ \nabla_\bsm \log \pi_{\mathrm{like}}(\bsy \, | \, \bsm) \; \nabla_\bsm \log \pi_{\mathrm{like}} (\bsy \, | \, \bsm)^T\right ]\\ {} & =  - \mathbb{E}_{\bsm, \bsy} \left [  \nabla_{\bsm} [\nabla_{\bsm} \log \pi_{\mathrm{like}}(\bsy \mid \bsm) ] \right ]. 
\end{split}
\label{eq:Jd}
\end{equation}
Both the gradient-based and Hessian-based expressions for $\bsJ^P_{\bsm}$ and $\bsJ^D_{\bsm}$ above can be used as the basis for Monte Carlo estimators of the matrices. 
The gradients of the log-prior and log-likelihood are generally cheaper to evaluate than the corresponding Hessians. 
However, the Hessian-based expressions may be advantageous 
when only a small number of Monte Carlo samples is available. 

\section{Methods}
\label{sec:methods}

As discussed in the introduction, we consider 
the computation of the Bayesian CRB for an imaging system whose 
unknown parameter $\bsm$ lies 
in a high-dimensional Euclidean 
space $\mathbb{R}^{D}$. 
We assume 
access to a set of i.i.d. samples from the ground-truth parameter distribution and 
knowledge of the likelihood function for the given inverse problem. This is formalized in the following problem statement. 

\begin{problem} Consider an imaging system with a high-dimensional parameter space. Given the corresponding likelihood function $\pi_{\mathrm{like}} (\bsy \mid \bsm)$ and a set of i.i.d. 
samples $\{ \bsm_i \}_{i=1}^N$ from an unknown prior 
distribution $\pi_{\mathrm{pr}}(\bsm)$, obtain an estimate $\hat{\bsV}_\bsm$ of the Bayesian CRB $\bsV_{\bsm}$. 
\end{problem}

One approach to address 
the above problem would be 
to model the prior score using SM and then estimate 
$\bsJ^D_{\bsm}$ and $\bsJ^P_{\bsm}$ 
in \eqref{eq:Jb_exact} 
using Monte Carlo methods. This approach, investigated 
in \cite{crafts2023}, admits 
theoretical guarantees and has been shown to perform 
well empirically on small-scale problems. However, there are two key issues with applying this approach to bound estimation for high-dimensional imaging systems. First, the high dimensionality of the problem setting complicates the accurate finite-sample estimation 
of $\bsJ^D_{\bsm}$ and $\bsJ^P_{\bsm}$ and the subsequent inversion of 
the Bayesian Fisher information matrix. 
Second, SM 
does not readily scale to high-dimensional systems. 
SSM \cite{song_2020_sliced} can help mitigate the computational cost associated with computing the SM 
objective while retaining consistency of the resulting score estimator. 
However, in practice, issues such as limited model capacity, insufficient 
training data, and 
a poorly behaved loss landscape can lead to artifacts in the score, and thus in the resulting bound estimate, as we 
demonstrate in Section~\ref{sec:results}. 

To address these challenges, 
we introduce two novel techniques to enable 
Bayesian CRB estimation for 
high-dimensional imaging systems. First, rather than forming the bound directly in the native high-dimensional parameter space, we compute 
the bound in the latent space of a variational autoencoder from a pre-trained foundation model. 
We then map the latent-space bound to the native parameter space using 
a change-of-variables formula \cite{bell_2013_detection}. 
Because the latent dimension 
is significantly smaller than the dimension of the native parameter space, this choice makes 
the estimation of the Bayesian Fisher information matrix and computation of its inverse, the Bayesian CRB, more tractable.   

Even with the reduced dimensionality provided by the autoencoder, 
we found in our numerical experiments that the use of SSM to estimate the prior term in the latent space introduced 
artifacts. Our second contribution addresses this issue by introducing a novel variant of the SSM objective. In particular, our approach takes advantage of the fact that the latent space of image autoencoders generally retains much of the spatial structure contained in the native parameter space. Accordingly, 
the standard SM 
and SSM objectives in the latent space can be viewed as discretizations of a loss function defined on 
a Bochner space with an $L^2$ codomain norm. We leverage this interpretation 
to derive the proposed SSM objective 
by replacing the $L^2$ codomain norm 
with an $H^1$ Sobolev norm. 
The resulting objective, like the original SSM objective, yields consistent estimators of the score. However, compared to the original SSM objective, it places a greater emphasis on suppressing 
high-frequency components of the score error that can appear as artifacts in the learned score model and the corresponding bound estimates. 

In what follows, we provide details of the formulation of the bound in the latent space, 
the proposed SM 
objective, and their integration to produce 
bound estimates in both the latent space and the original parameter space.

\subsection{Bound Formation in the Latent Space}

Let $\phi^e$ and $\phi^d$ denote 
the encoder and decoder maps, respectively, of 
an autoencoder. For concreteness, we assume that the autoencoder is a variational autoencoder, with the output of $\phi^e$ being the mean and pointwise standard deviation of a Gaussian distribution 
in the latent space. Accordingly, the 
maps $\phi^e$ and $\phi^d$ 
have the form $\phi^e: \mathbb{R}^D \to \mathcal{Z} \times \mathcal{Z}$ and $\phi^d: \mathcal{Z} \to \mathbb{R}^D$, where $\mathcal{Z}$ is a 
$Z$-dimensional latent space. We further assume that the latent space has spatial structure. Specifically, 
$\mathcal{Z}$ is identified with a vectorization of the space $\mathbb{R}^{\textcolor{blue}{N_c} 
\times Z_1 \times Z_2}$ 
and therefore has dimension $Z = \textcolor{blue}{N_c}
\, Z_1 \, Z_2$, where \textcolor{blue}{$N_c$} 
denotes the number of channels, 
and $Z_1$ and $Z_2$ denote the spatial dimensions. 
In the remainder of this section, we set 
\textcolor{blue}{$N_{c} 
= 1$} for simplicity of exposition. Note that the autoencoder used here can be any pre-trained foundation-model variational autoencoder and 
need not 
be fine-tuned on the 
dataset of interest; see Section \ref{sec:num_stud} for details. 

Given the autoencoder, the inverse problem and corresponding bound estimation problem can be reformulated in the 
latent space of the autoencoder. 
Under this reformulation, the latent-space prior distribution is given by 
\begin{equation}
\label{eq:latent_prior}
\pi_{\mathrm{ltnt}}(\bsz) \triangleq \int_{\mathbb{R}^D} \mathcal{N}\left(\bsz; \phi^e(\bsm)_1, \mathrm{diag}(\phi^e(\bsm)_2^2) \, \right) \; d\pi_{\mathrm{pr}}(\bsm).
\end{equation}
Here, $\bsz\in\mathcal{Z}$ is the latent variable, and $ \phi^e(\bsm)_1$ and $\phi^e(\bsm)_2$ denote 
the mean and standard deviation, 
respectively, of the encoded latent Gaussian corresponding to $\bsm$. 
Above, $\mathrm{diag}(\cdot)$ returns a diagonal matrix whose diagonal entries are given by its argument. 
We emphasize that the latent-space density is unknown because 
$\pi_{\mathrm{pr}}$ is not analytically accessible in our setting; the score of $\pi_{\mathrm{ltnt}}$ must therefore be learned from data. 

To construct the corresponding likelihood, we assume for the sake of concreteness that the image-space likelihood follows an 
additive Gaussian noise model, i.e., $\pi_{\mathrm{like}}(\bsy \mid \bsm) = \mathcal{N}(\bsy; f(\bsm), \bsSigma)$, where $\bsSigma \in \mathbb{R}^{K \times K}$ is the noise covariance matrix 
and $f: \mathbb{R}^D \to \mathbb{R}^K$ is the forward operator. The corresponding latent-space likelihood, denoted by $\pi_{\mathrm{dlike}}$, can then be written as
\begin{equation}
    \label{eq:latent_like}
    \pi_{\mathrm{dlike}}(\bsy \mid \bsz) = \mathcal{N}(\bsy; f(\phi^d(\bsz)), \bsSigma). 
\end{equation}

As detailed in the following subsections, we use the latent-space prior $\pi_{\mathrm{ltnt}}$ and the latent-space likelihood $\pi_{\mathrm{dlike}}$ to construct an estimate $\hat{\bsV}_{\bsz}$ of the Bayesian CRB $\bsV_{\bsz} \in \mathbb{R}^{Z \times Z}$ for the latent-space parameter $\bsz$. The bound estimate can be used directly for optimal experimental design. Alternatively, a bound on the image-space parameter $\bsm$ can be obtained via the change-of-variables formula for the Bayesian CRB \cite{bell_2013_detection} as
\begin{equation}
\begin{split}
    {}&\mathbb{E}_{\bsm, \bsy} \left [ (\hat{\bsm}(\bsy) - \bsm)  (\hat{\bsm}(\bsy) - \bsm)^T \right ] \succeq \bsC \bsV_{\bsz} \bsC^T,\\ {}& \bsC \triangleq \mathbb{E}_{\bsz} \left [\nabla_{\bsz} \phi^d(\bsz) \right].
\end{split}
\end{equation}

In our approach, we construct a Monte Carlo estimator $\hat{\bsC}$ of the matrix $\bsC$: 
\begin{equation}
    \label{eq:cov_formula}
       \hat{\bsC} \triangleq \frac{1}{N_s} \sum_{i=1}^{N_s} \nabla_{\bsz} \phi^d(\bsz_i), \quad \bsz_i \sim \pi_{\mathrm{ltnt}}(\bsz),
\end{equation}
where $N_s$ denotes the number of Monte Carlo samples. 
The Jacobians $\nabla_{\bsz} \phi^d(\bsz_i)$ are computed using forward-mode automatic differentiation to leverage the smaller latent-space dimension 
relative to the native image-space dimension.
The final image-space bound estimate $\hat{\bsV}_\bsm$ is then defined as 
\begin{equation}
    \label{eq:image_space_cov_estimate}
    \hat{\bsV}_\bsm \triangleq \hat{\bsC} \hat{\bsV}_\bsz \hat{\bsC}^T. 
\end{equation}
Note that, while explicitly forming $\hat{\bsV}_\bsm$ may incur 
prohibitive computational 
and memory costs in high-dimensional settings, 
individual entries of $ \hat{\bsV}_\bsm$ can be computed 
efficiently as
\begin{equation*}
    (\hat{\bsV}_{\bsm})_{ij} = \sum_{k=1}^{Z} \hat{\bsC}_{ik} \sum_{l=1}^{Z} (\hat{\bsV}_\bsz)_{kl} \hat{\bsC}_{jl}. 
\end{equation*}

\subsection{Sliced Score Matching in Bochner Spaces}

The Bayesian CRB in the latent space depends on 
the scores of both the latent-space likelihood $\pi_{\mathrm{dlike}}$ and the latent-space prior $\pi_{\mathrm{ltnt}}$. To estimate the latent-space prior score, we introduce a novel variant of SSM. 
SSM is 
based on minimizing the Fisher divergence between 
a score model and the ground-truth score. 
In the latent space, the Fisher divergence objective in \eqref{eq:fisher} can be written as 
\begin{equation}
L(\bstheta) = \frac{1}{2} \mathbb{E}_{\bsz} \left [ \left \| s_{\bstheta}(\bsz) - \nabla_{\bsz} \log \pi_{\mathrm{ltnt}} (\bsz) \right \|_2^2 \right],
\label{eq:latent_fisher}
\end{equation}
where 
$s_{\bstheta}(\mathbf{z})$ 
is the score model and $\nabla_{\bsz} \log \pi_{\mathrm{ltnt}} (\bsz)$ 
is the ground-truth score. Since $\mathcal{Z}$ is assumed to have spatial structure, we can interpret $\bsz$ as the discretization of a function $z(\bsx)$, where 
$\bsx \in \Omega \subset \mathbb{R}^2$.
This interpretation motivates the proposed approach and
can be formalized via Bochner spaces. 

Bochner spaces generalize scalar-valued 
$L^p$ function spaces to functions taking values in a Banach space 
and are widely used in the analysis of partial differential equations (see, e.g., Section 5.9 of \cite{evans2022partial}). Let $H_1$ and $H_2$ denote separable Hilbert spaces, and let $\mu$ be a measure on $H_1$. Then, the Bochner space $L^2(H_1, \mu; H_2)$ is the space of all equivalence classes of Borel-measurable, square-integrable mappings $f: H_1 \rightarrow H_2$ such that
\vspace{-0.6mm}
\begin{equation*}
\| f \|_{L^2(H_1, \mu; H_2)} \triangleq \left({\int_{H_1} 
\| f(z) \|^2_{H_2} \; d\mu(z)}\right)^{1/2} < \infty. 
\end{equation*}
\vspace{-0.6mm}
From the Bochner space perspective, the Fisher divergence objective in \eqref{eq:latent_fisher} can be understood as the
discretization of an objective based on a Bochner-space norm. 
Specifically, the corresponding Bochner space is $L^2( L^2(\Omega), \mu_{\mathrm{ltnt}}; 
L^2(\Omega))$, where $\mu_{\mathrm{ltnt}}$ denotes 
the measure corresponding to $\pi_{\mathrm{ltnt}}$ before discretization. In the proposed method, the $L^2$ codomain norm is replaced with an $H^1$ Sobolev 
norm, leading to the Bochner space $L^2( L^2(\Omega), \mu_{\mathrm{ltnt}}; 
H^1(\Omega))$.

After spatial discretization, 
we obtain the following weighted version of the Fisher divergence objective:
\vspace{-0.5mm}
\begin{equation}
\label{eq:fisher_latent_weighted}
\begin{split}
    {} & L(\bstheta; H^1) \triangleq \frac{1}{2} \mathbb{E}\left [ \left \| s_{\bstheta}(\bsz) - \nabla_{\bsz} \log \pi_{\mathrm{ltnt}} (\bsz) \right \|_\mathbf{W}^2 \right], \\
    {} &\bsW \triangleq \mathbf{I} + \mathbf{L},
\end{split}
\end{equation}
\vspace{-0.5mm}
where $\| \cdot \|_\mathbf{W}$ denotes the weighted norm induced by $\mathbf{W}$, and $\mathbf{L} \in \mathbb{R}^{Z \times Z}$ is a discretization of the two-dimensional Laplacian operator. We implement $\mathbf{L} \bsz$ efficiently using batched fast Fourier transforms; see Section \ref{sec:num_stud} for details. Note that this weighted objective is consistent with the original Fisher divergence objective in the sense that $L(\bstheta; H^1) = 0$ if and only if $L(\bstheta) = 0$. 
However, we observed empirically that using $L(\bstheta)$ can lead to 
high-frequency errors in the learned score, 
which cause artifacts in the resulting bound estimates. The $L(\bstheta; H^1)$ objective places comparatively greater emphasis on reducing these high-frequency error components. 
 
As with the original Fisher divergence objective, \eqref{eq:fisher_latent_weighted} cannot be directly used to learn the score because it depends on 
the \textit{a priori} unknown quantity $\nabla_{\bsz} \log \pi_{\mathrm{ltnt}} (\bsz)$. 
However, the same 
arguments used to derive the SM 
and SSM objectives can be extended to obtain 
a weighted variant of the SSM objective from \eqref{eq:fisher_latent_weighted}. This is formalized by the following theorem. 

\begin{theorem}
Assume that the latent-space distribution $\pi_{\mathrm{ltnt}}$ and the score model $s_{\bstheta}$ satisfy Assumptions 1 and 3 in \cite{song_2020_sliced}. Then $L(\bstheta; H^1) = J_{\mathrm{sl}}(\bstheta; H^1) + C$, where 
$$
J_{\mathrm{sl}}(\bstheta; H^1) \triangleq \mathbb{E}_{\bsv} \, \mathbb{E}_{\bsz}  \left [ \frac{1}{2} \left \| s_{\bstheta}(\bsz) \right \|_\bsW^2 + \bsv^T \bsW \left [\nabla_{\bsz} s_{\bstheta}(\bsz) \right] \bsv \right]. 
$$
Further, if 
the score model family is well specified and identifiable, the parameter space is compact, and the score model and ground-truth score satisfy Lipschitz conditions (Assumptions 4, 6, and 7 in \cite{song_2020_sliced}), then the objective
\begin{equation}
\begin{split}
\hat{J}_{\mathrm{sl}}(\bstheta; H^1) \triangleq \frac{1}{N} \sum_{i=1}^N& \left ( \frac{1}{2} \left \| s_{\bstheta}(\bsz_i) \right \|_{\bsW}^2 \phantom{\sum_{j=1}^M} \right.\\ {} & \left. + \frac{1}{M} \sum_{j=1}^M \, \bsv_{ij}^T \bsW \left [\nabla_{\bsz} s_{\bstheta}(\bsz_i) \right] \bsv_{ij} \right),
\end{split}
\label{eq:ssm_empirical_weighted}
\end{equation}
yields a consistent estimator of the score. In particular,
$$
\hat{\bstheta} = \argmin_{\bstheta \in \bsTheta} \hat{J}_{\mathrm{sl}}(\bstheta; H^1) 
$$
converges in probability to $\bstheta^*$ as $N \to \infty$, where $\bstheta^*$ denotes 
the ground-truth parameter satisfying 
$s_{\textcolor{blue}{\bstheta^*}}(\bsz) 
= \nabla_{\bsz} \log \pi_{\mathrm{ltnt}}(\bsz)$. 
\end{theorem}

\begin{proof}
The proof follows that of Theorem 2 in \cite{song_2020_sliced} with only minor modifications. The details are therefore omitted. 
\end{proof}

\subsection{Putting the Pieces Together}

We now describe how the proposed Bochner-space SSM objective is used to estimate the latent-space Bayesian CRB $\bsV_{\bsz}$ via 
 the decomposition of the Bayesian Fisher information matrix given in \eqref{eq:Jb_exact}.  

To estimate the prior term $\bsJ^P_{\bsz}$, we first train a score model using the Bochner-space SSM objective in \eqref{eq:ssm_empirical_weighted}. 
The latent-space samples $\bsz_i$ used to train the model are drawn 
from $\pi_{\mathrm{ltnt}}(\bsz)$ by 
encoding the training samples $\{ \bsm_i \}_{i=1}^N$ 
and then sampling from the corresponding latent Gaussian distributions. We implement 
the vector-valued score model $s_{\bstheta}(\bsz)$ as the gradient of a scalar-valued function $e_{\bstheta}(\bsz)$: 
$$
s_{\bstheta}(\bsz) = \nabla_{\bsz} e_{\bstheta}(\bsz). 
$$
This implementation ensures that the score model, like the ground-truth score, is a conservative vector field. For the scalar-valued function $e_{\bstheta}(\bsz)$, we use a least-squares-residual-based parameterization \cite{hertrich2025learning}: 
\begin{equation}
\label{eq:energy_net_lsr_param}
e_{\bstheta}(\bsz) = -\frac{1}{2} \left \|\eta_{\bstheta}(\bsz) - \bsz \right \|_2^2,
\end{equation}
where $\eta_{\bstheta}: 
\mathbb{R}^Z \to \mathbb{R}^Z$ is a neural network. 

Once the latent-space score has been learned, the $\bsJ^P_{\bsz}$ term can be estimated using Monte Carlo methods as
$$
\hat{\bsJ}^P_{\bsz} \triangleq \frac{1}{N}\sum_{i=1}^{N} s_{\hat{\bstheta}}(\bsz_i) \, s_{\hat{\bstheta}}(\bsz_i)^T,
$$
where $s_{\hat{\bstheta}}$ denotes 
the learned score model. 
For each $\bsm_i$, the latent sample $\bsz_i$ is drawn according to 
$\bsz_i \sim \mathcal{N}\left(
\phi^e(\bsm_i)_1, \mathrm{diag}(\phi^e(\bsm_i)_2^2) \, \right)$. 

To estimate the term $\bsJ_{\bsz}^D$, 
we again use a Monte Carlo estimator: 
\begin{equation}
\label{eq:Jd_estimate}
\hat{\bsJ}^D_{\bsz} \triangleq \frac{1}{N}\sum_{i=1}^{N} \bsg_i \, \bsg_i^T, \quad \bsg_i \triangleq \nabla_{\bsz} \log \pi_{\mathrm{dlike}}(\bsy \mid \bsz_i).
\end{equation}
The latent-space likelihood scores $\bsg_i$ can be computed in two ways, depending on the application. First, automatic differentiation can 
be used to differentiate 
through both the forward model $f$ and the decoder $\phi^d$ in the latent-space likelihood defined in 
\eqref{eq:latent_like}. Alternatively, 
the chain rule gives 
$$
\bsg_i = \left[\nabla_{\bsz} \phi^d(\bsz_i) \right]^T \nabla_{\bsm} \log \pi_{\mathrm{like}}(\bsy \mid \bsm_i), \quad \bsm_i = \phi^d(\bsz_i). 
$$
Under this second approach, the 
image-space likelihood score $\nabla_{\bsm} \log \pi_{\mathrm{like}}(\bsy \mid \bsm_i)$ can first 
be computed 
using analytic expressions, adjoint methods, or other gradient computation methods. 
The latent-space likelihood score $\bsg_i$ is then computed using reverse-mode automatic differentiation as the vector--Jacobian product 
$$
\bsg_i = \nabla_{\bsz} \left[\phi^d(\bsz_i)^T \, \nabla_{\bsm} \log \pi_{\mathrm{like}}(\bsy \mid \bsm_i) \right]. 
$$

After computing 
$\hat{\bsJ}^P_{\bsz}$ and $\hat{\bsJ}^D_{\bsz}$, an estimate of the Bayesian Fisher information matrix is 
formed as 
\begin{equation}
\hat{\bsJ}_{\bsz} = \hat{\bsJ}^P_{\bsz} + \hat{\bsJ}^{D}
_{\bsz}.
\label{eq:estimate_fisher}
\end{equation}
A pseudoinverse may be used to 
obtain a numerically stable 
estimate of the Bayesian CRB. 
However, in our experiments, 
$\hat{\bsJ}_{\bsz}$ was always invertible, so 
we set $\hat{\bsV}_{\bsz} = \hat{\bsJ}_{\bsz}^{-1}$. As mentioned above, this bound estimate can 
be used directly for experimental design or transformed into an estimate of 
the image-space bound $\bsV_\bsm$ using \eqref{eq:image_space_cov_estimate}. 

\section{Numerical Study}
\label{sec:num_stud}

To validate the proposed approach, we conducted a numerical study using a stylized inverse problem inspired by quantitative photoacoustic computed tomography (qPACT) 
\cite{cox2009challenges, tarvainen2024quantitative}. 
In the following subsections, we provide details of 
the problem setting and 
the study implementation. 

\subsection{Problem Setting}

Quantitative PACT is a multi-physics imaging modality that aims to reconstruct images of 
the underlying molecular or optical properties of the tissue being imaged \cite{wang2012biomedical, xu2006photoacoustic}. The contrast mechanism in qPACT is based on the photoacoustic effect, in which 
absorbed electromagnetic energy is converted into acoustic waves. In qPACT, a sequence 
of short laser pulses illuminates the tissue, and ultrasonic transducers surrounding the tissue detect the resulting acoustic waves. 
These measurements are then used to reconstruct the image of interest. 

In this work, we considered the estimation of the optical absorption coefficient $\mu_a
$ in a stylized, multi-illumination qPACT 
experiment in two spatial dimensions. 
For 
training, 
we used two-dimensional absorption coefficient maps $\bsmu_a \in \mathbb{R}^{D_1 \times D_2}$ derived from the stochastic three-dimensional numerical breast phantoms introduced in \cite{park2023stochastic, park2026virtual, chen2025benchmarking}. Specifically, 
$2{,}684$ three-dimensional numerical breast phantoms were used to generate $40{,}000$ two-dimensional absorption coefficient maps. 
Each map was obtained by randomly selecting a breast phantom, extracting a random $2 \, \mathrm{mm}$-thick coronal slice from the selected phantom, assigning stochastic functional and optical properties as in 
\cite{park2023stochastic}, and averaging the $\mu_a$ values across the slice thickness. 
The resulting images were 
of size $1024 \times 1024$, leading to an inverse problem with more than 
one million unknown parameters. Uncurated samples from this dataset are displayed in the first row of Fig. \ref{fig:recon_results}. 

For the forward model, 
we assumed that the underlying tissue was 
acoustically homogeneous and lossless, allowing 
acoustic wave propagation to 
be modeled using 
the circular Radon transform \cite{poudel2019survey}. We also assumed that the photoacoustic efficiency of the tissue was 
known and adopted 
a simplified 
optical model.
These assumptions led 
to the following linear-Gaussian likelihood: 
\begin{equation}
\pi_{\mathrm{like}} (\bsy \mid \bsm) \triangleq \mathcal{N}(\bsy; \bsd(\bsm), \tau^2 \mathbf{I}),
\end{equation}
where $\bsm = \mathrm{vec} (\bsmu_a) \in \mathbb{R}^D$, $\tau = 0.01$, and $\bsd: 
\mathbb{R}^D \to \mathbb{R}^K$ is the linear(ized) forward model, defined as
\begin{equation}
\label{eq:qpact_like}
\bsd(\bsm) \triangleq \mathrm{vec}(\{\mathcal{R}( \mathbf{D}_i \, \bsm), i = 1, \dots, I \}).
\end{equation}
Here, $\mathcal{R}$ denotes 
the two-dimensional circular Radon transform, and $\bsD_i \in \mathbb{R}^{D \times D}$ is a diagonal weighting matrix representing 
the fluence distribution induced by the $i$th illumination. We considered 
the ``contiguous'' and ``interlaced'' illumination 
schemes introduced in \cite{crafts2024bayesian}, each of which used $I=4$ illuminations. For each scheme, the 
diagonal entries of the matrices $\{ \bsD_i \}_{i=1}^{I}$, 
corresponding to the fluence maps, were computed 
by solving a diffusion partial differential equation (PDE) that models 
light propagation under each 
illumination \cite{wang2012biomedical}. 
The mean of the training samples $\{ \bsm_i \}_{i=1}^N$ 
was used as the absorption coefficient in these 
simulations, and the diffusion equation 
was solved using the finite element 
method \cite{AlnaesEtal2015}. 
Further details are provided in \cite{crafts2024bayesian}. 
Figure~\ref{fig:fluence_figure} shows 
the diagonals of the $\bsD_i$ matrices for both illumination schemes. 

\begin{figure}
    \centering
    \includegraphics[width=\linewidth]{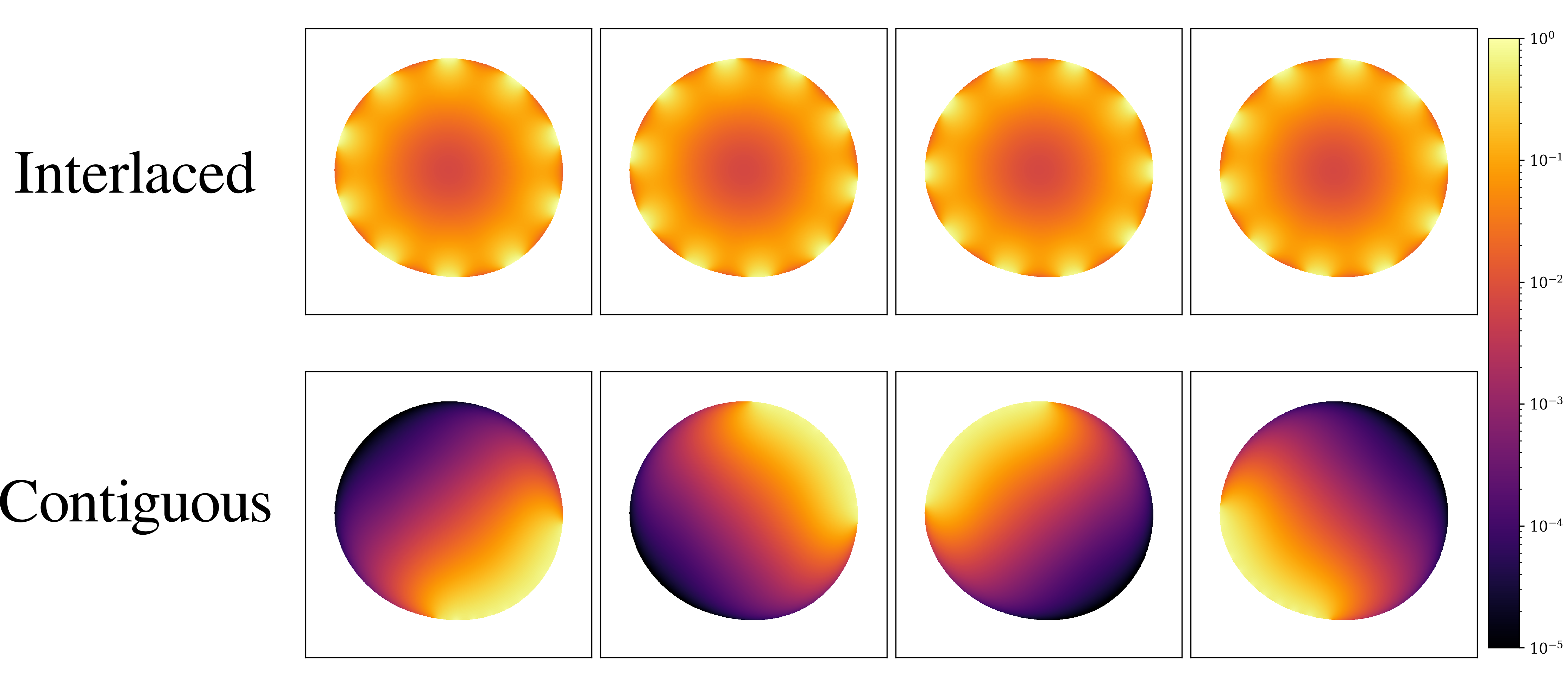}
    \caption{
    Fluence maps simulated under 
    the interlaced (top) and contiguous (bottom) 
    illumination schemes used in the numerical studies. These maps correspond to 
    the diagonals of the weighting matrices $\{\bsD_i\}_{i=1}^{I}$ in \eqref{eq:qpact_like}, with $I=4$. 
    The white regions indicate locations that contained only water across the dataset, 
    i.e., regions outside of the breast phantoms; 
    the fluence was set to zero in these regions. 
    Although the two illumination schemes produce different spatial fluence patterns, 
    both exhibit 
    significant light attenuation near 
    the center of the domain due to the highly scattering nature of biological tissue. }
    \label{fig:fluence_figure}
\end{figure}


\subsection{Implementation of Proposed Approach}

The proposed approach used 
a pre-trained autoencoder to facilitate bound computation. In our experiments, we used 
the variational autoencoder from 
Stable Diffusion v1, a latent diffusion modeling framework \cite{rombach2022highresolution}. In particular, we used the model architecture designed for RGB images of size $3 \times 512 \times 512$, with a latent space of size 
$4 \times 64 \times 64$. To accommodate the single-channel $1024 \times 1024$ data, 
we 
modified the autoencoder by incorporating a fixed down-sampling and broadcasting layer at the start of the encoder and a fixed up-sampling and averaging layer at the end of the decoder. 
Because the 
values of $\mu_a$ 
span several orders of magnitude, which can be difficult for the autoencoder to handle, we prepended 
a pointwise logarithmic transformation to the encoder 
and shifted and scaled the transformed 
data 
into $[-1, 1]^D$. The inverse of this transformation was 
appended to the decoder. Notably, no fine-tuning was performed. Figure \ref{fig:recon_results} shows uncurated $\bsmu_a$ maps 
from our dataset 
together with their corresponding autoencoder reconstructions. 
The reconstructions exhibit high visual fidelity without noticeable 
distortions, and 
the average relative squared reconstruction error 
over the entire dataset was $7.67\%$. 

\begin{figure}
    \centering
    \includegraphics[width=\linewidth]{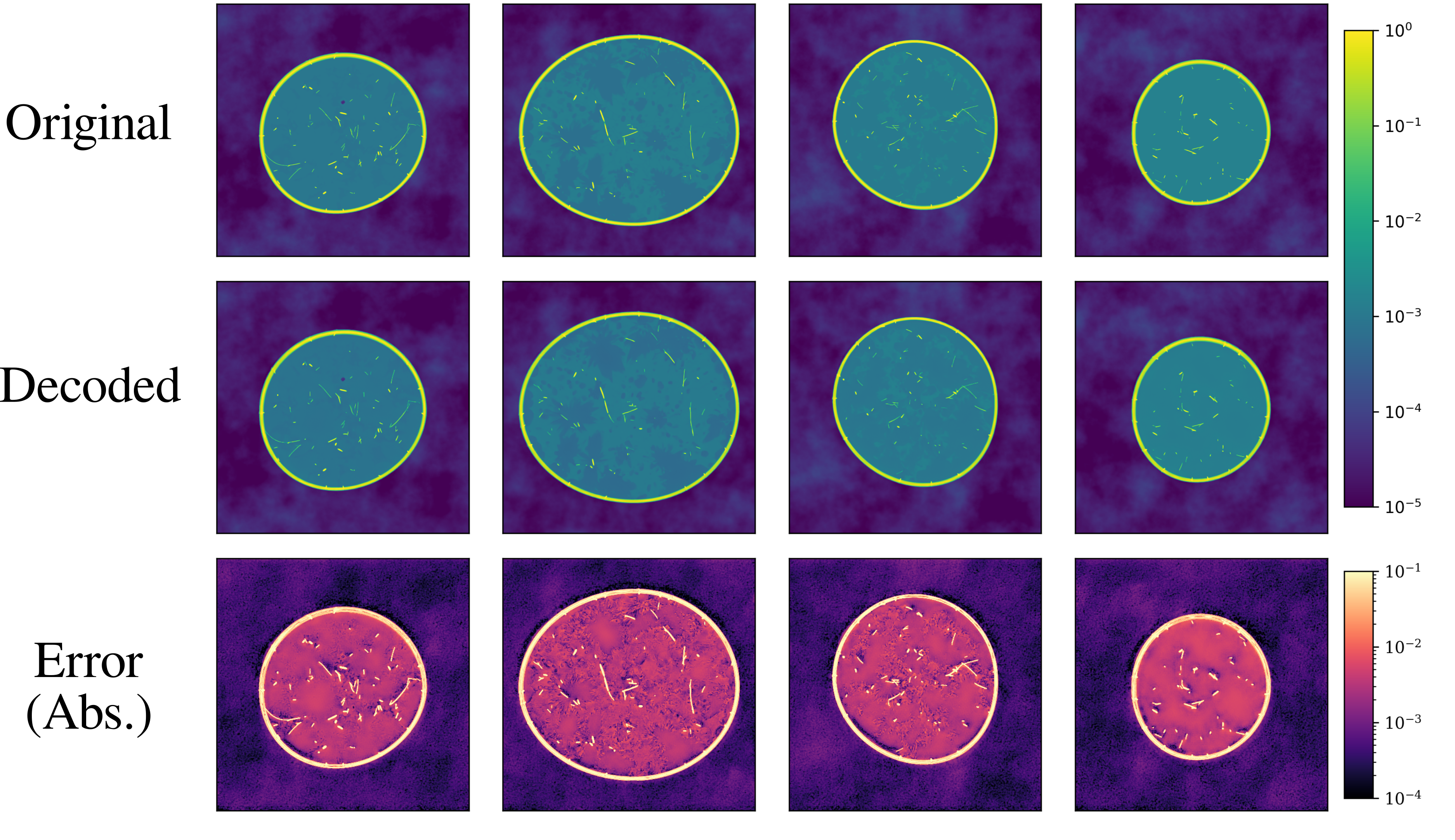}
    \caption{Uncurated $\bsmu_a$ maps from our dataset (top), 
    the corresponding reconstructions obtained using the autoencoder (middle), and the associated 
    error maps (bottom). 
    Despite not being trained or fine-tuned on our dataset, the autoencoder produced accurate reconstructions without noticeable visual distortions.} 
    \label{fig:recon_results}
\end{figure}

We implemented both the original and proposed weighted SSM objective functions in the latent space using \eqref{eq:ssm_empirical_weighted} with $\mathbf{W}=\mathbf{I}$ and $\mathbf{W}=\mathbf{I}+\mathbf{L}$, respectively. The two-dimensional Laplacian operator $\mathbf{L}$ was implemented in Fourier space with symmetric mirror padding 
to prevent aliasing. 
We used the same network architecture and training procedure 
for both the original and weighted SSM objectives. The network $\eta_{\bstheta}$ in \eqref{eq:energy_net_lsr_param} was implemented as 
a four-layer convolutional neural network with 
Softplus activations 
and a kernel size of $3$. The three hidden layers contained 
$16$, $32$, and $16$ channels, respectively, resulting in a 
total of 
$10{,}437$ model parameters. The SSM objective in \eqref{eq:ssm_empirical_weighted} was implemented with $M = 10$. 
We optimized the network weights using the AdamW optimizer \cite{loshchilov2017decoupled} with a weight decay of $10^{-4}$ and a learning rate of $2 \times 10^{-4}$. Gradient clipping and exponential moving averages of the network weights were incorporated for stability. The 
networks trained 
with the original and weighted SSM objectives 
were each optimized 
for $450{,}000$ 
iterations. 

\begin{figure}
    \centering
    \includegraphics[width=\linewidth]{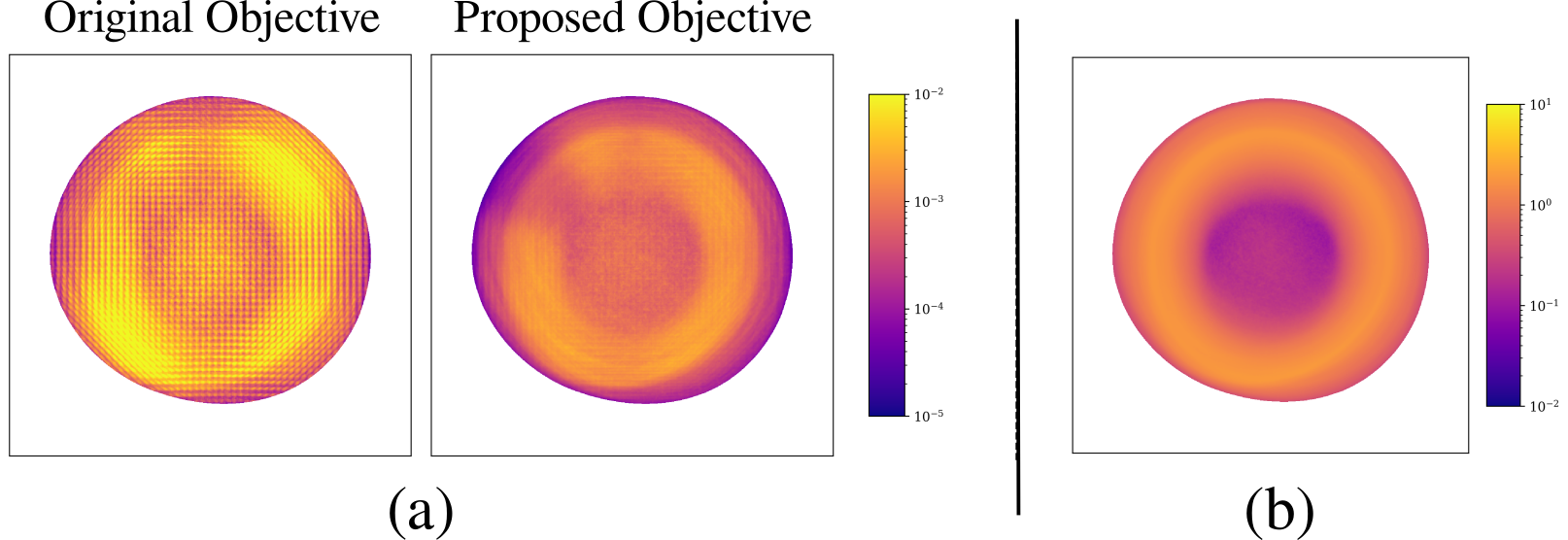}
    \caption{(a) Diagonals of the estimates of $\bsV^P_{\bsm} $ in \eqref{eq:jp_pushforward} obtained using 
    the original (left) and weighted (right) SSM objectives. (b) 
    Pointwise variance of the samples $\{ \bsm_i \}_{i=1}^N$. 
    In both panels, 
    regions containing only 
    water 
    across the dataset were masked out. 
    Both 
    $\bsV_{\bsm}^P$ estimates exhibit 
    similar global structure, 
    consistent with that of the pointwise variance. However, the estimate obtained using the original SSM objective contains 
    high-frequency artifacts 
    that are absent from the estimate obtained 
    using the proposed objective.  }
    \label{fig:jp_estimates}
\end{figure}

Other implementation details are as follows. We computed the latent-space likelihood gradients in \eqref{eq:Jd_estimate} using reverse-mode automatic differentiation through both the likelihood $\pi_{\mathrm{like}}$ and the decoder. 
We estimated the matrix in the change-of-variables formula 
in \eqref{eq:cov_formula} 
using $N_s = 4{,}000$ samples and 
computed the Jacobians $\nabla_{\bsz} \phi^d(\bsz)$ using forward-mode automatic differentiation.

\section{Results}
\label{sec:results}

\begin{figure}
    \centering
    \includegraphics[width=\linewidth]{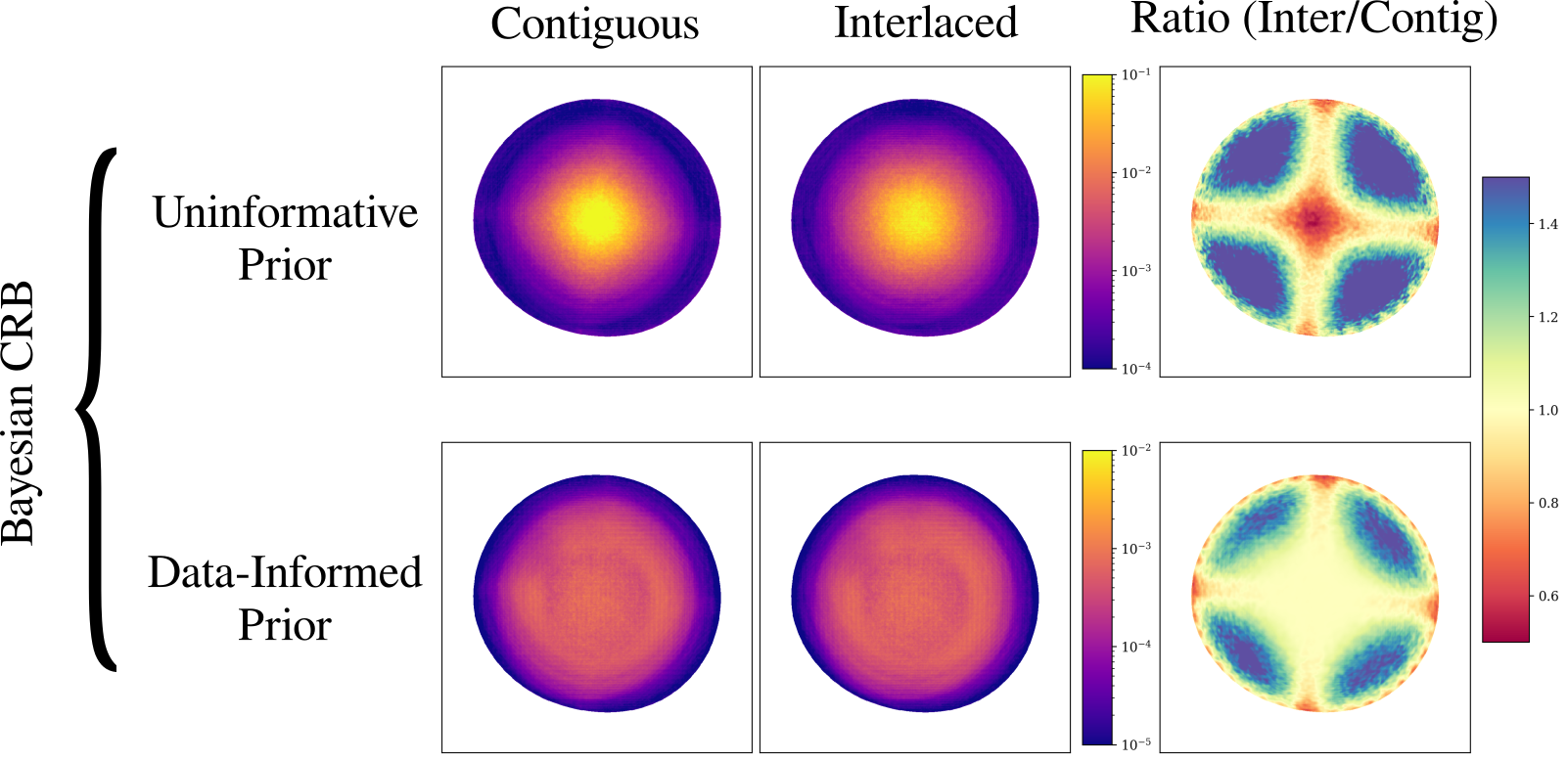}
    \vspace{-5mm}
    \caption{Estimates of the diagonals of the Bayesian CRB $\bsV_{\bsm}$ (bottom) 
    and the Bayesian CRB under an uninformative prior $\bsV_{\bsm}^D$ (top; 
    see 
    \eqref{eq:jd_pushforward}) for the contiguous and interlaced illumination 
    schemes. The final column shows the pointwise ratio of the two bounds, with red indicating locations 
    where the interlaced scheme performs better and 
    blue indicating locations where the contiguous scheme is superior. 
    The spatial structure of the bound under the uninformative 
    prior is consistent with that of 
    the diagonal weighting matrices shown in Fig. \ref{fig:fluence_figure}. 
    Incorporating the data-informed prior 
     changes both the spatial structure of the bounds and their relative strengths because 
     the prior strongly informs the parameters near 
     the center of the domain; 
     cf. Fig. \ref{fig:jp_estimates}(b).}
     
    \label{fig:jd_bound_estimates}
\end{figure}

\begin{table*}[tb]
    \caption{
    Trace and log-determinant 
    of the estimated 
    Bayesian CRBs 
    (lower is better). 
    With the uninformative prior, the ranking 
    of the two illumination 
    schemes depends on the chosen metric: 
    the latent-space bound favors 
    the contiguous 
    scheme, whereas 
    the parameter-space bound favors 
    the interlaced 
    scheme. The incorporation of the learned prior resolves this discrepancy, 
    with the contiguous 
    scheme favored by all reported 
    metrics. }
    \centering
    \begin{tabular}{c|ccc | ccc }
    \rowcolor{my_gray} Design Scheme & \multicolumn{3}{c|}{Uninformative Prior} & \multicolumn{3}{c}{Data-Informed Prior} \\
   \rowcolor{my_gray} & Trace $\bsV_{\bsz}^D$ & Log Det. $\bsV_{\bsz}^D$ & Trace $\bsV_{\bsm}^D$      & Trace $\bsV_{\bsz}$ & Log Det. $\bsV_{\bsz}$ & Trace $\bsV_{\bsm}$  \\
        
    \hline
    Interlaced    & $217819.2410$ & $-33163.6517$ &  $\mathbf{3065.9900}$ & $11.2323$ & $-131923.8601$ & $147.4793$ \\
    Contiguous    & $\mathbf{167677.041}$ & $\mathbf{-35779.3460}$  &  $3844.3887$ & $\mathbf{10.9382}$ & $\mathbf{-132472.7493}$ & $\mathbf{134.2040}$ 
    \end{tabular}
    \label{tab:metric_results}
\end{table*}

To isolate the effect 
of the learned prior on the Bayesian CRB, we first examine the estimates of $\bsJ^P_{\bsz}$ 
obtained using both the original and proposed SSM objectives. These are mapped to image space using the change-of-variables formula in \eqref{eq:image_space_cov_estimate}. 
Specifically, we computed estimates of  
\begin{equation}
\label{eq:jp_pushforward}
\bsV^P_{\bsm} \triangleq \bsC \, \bsV^{P}_{\bsz} \, \bsC^T, \quad \bsV^P_{\bsz} \triangleq \left( \bsJ_{\bsz}^P \right)^{-1}.
\end{equation}
Figure \ref{fig:jp_estimates} shows the estimates of $\bsV^P_{\bsm} $ 
obtained using 
the original and proposed SSM objectives, together with the pointwise variance of the samples in the dataset $\{ \bsm_i \}_{i=1}^{N}$. Three observations are noteworthy. 
First, both 
bound estimates exhibited 
spatial structure similar to that of the sample pointwise variance, indicating 
that the bound estimates captured 
features present in the training distribution. 
Second, both bound estimates were 
several orders of magnitude smaller than the pointwise sample variance. This is 
expected because the pointwise variance of the samples provides an upper bound on the diagonal of $\bsV^P_{\bsm}$ 
that is attained only when the $\bsm$ distribution is Gaussian (see, e.g., \cite{bell_2013_detection}), while the distribution of the training data is highly non-Gaussian. Finally, although 
the estimates obtained with the two different objectives exhibited 
similar global structure, the estimate obtained using the original SSM objective contained 
high-frequency artifacts that were absent from 
the estimate obtained using the proposed objective. This observation is 
consistent with the stronger penalty 
that the proposed objective places on 
high-frequency components of the error in the learned score. 

As with the prior term, we  
examined the impact of the data term $\bsJ^D_{\bsz}$ by computing  
estimates of
\begin{equation}
\label{eq:jd_pushforward}
\bsV_{\bsm}^D \triangleq \bsC \, \bsV_{\bsz}^D \, \bsC^T, \quad \bsV_{\bsz}^D \triangleq \left ( \bsJ_{\bsz}^D \right)^{-1}. 
\end{equation}
Here, $\bsV_{\bsz}^D$ 
is the inverse of the expected Fisher information associated with 
the likelihood 
and corresponds to the latent-space Bayesian CRB under  
an uninformative prior. The matrix $\bsV_{\bsm}^D$ is the corresponding image-space bound obtained using the change-of-variables formula. Figure \ref{fig:jd_bound_estimates} shows the diagonals of both the estimated uninformative-prior Bayesian CRB $\bsV_{\bsm}^D$  
and the learned-prior Bayesian CRB $\bsV_{\bsm}$. 
For the latter, the prior score was learned using the proposed SSM objective.
As can be seen, the spatial structure of the bounds under the uninformative prior was 
consistent with that of 
the fluence maps in Fig. \ref{fig:fluence_figure}. In particular, both illumination 
schemes led 
to large bound values 
near the center of the domain, where signal attenuation was 
strongest. The 
relative performance 
of the two schemes, as shown in the final 
column of Fig.~\ref{fig:jd_bound_estimates}, was 
consistent with 
the contiguous design concentrating 
its illuminations in four distinct regions. Incorporation of the prior term $\bsJ_{\bsz}^P$ altered 
both the spatial structure of the bounds and the relative strengths of the two designs. In particular, the Bayesian CRB under the data-informed prior was dominated by the prior term near 
the center of the domain 
because 
signal attenuation was severe and the training samples exhibited 
relatively low variance in this region. 

Table \ref{tab:metric_results} reports 
the trace and log-determinant 
of the estimates of the latent-space bounds $\bsV_{\bsz}^D$ and $\bsV_{\bsz}$, together with the trace of 
their counterparts in the original parameter space. Note that the trace and log-determinant correspond to the A-optimal and D-optimal design criteria, respectively, and are commonly used metrics in experimental design \cite{pukelsheim_1993_optimal}. The results are consistent with 
the findings in Fig. \ref{fig:jd_bound_estimates} and further demonstrate that the learned prior substantially affects 
the structure of the bound. In particular, the incorporation of the data-informed prior changed 
the ranking 
of the two design schemes when 
the trace of the bound in parameter space was used as the design criterion. 
This finding suggests that the prior can play an important role in data-informed experimental design. 

\section{Discussion and Conclusion}
\label{sec:discuss_conclude}

In this work, we introduced a new framework to facilitate 
data-driven Bayesian CRB estimation for high-dimensional imaging systems. The framework uses a pre-trained variational autoencoder to reformulate the bound estimation problem in the autoencoder latent space, which dramatically decreases the dimensionality of the bound estimation problem. A novel SM 
variant, 
SSM 
in Bochner spaces, was 
also introduced to facilitate learning the score of the latent-space prior distribution. By formulating the SSM objective within a Bochner space that accounts for the spatial regularity of the field, we effectively penalized discrepancies in the spatial gradients of the score. The results demonstrated that this approach suppressed 
high-frequency artifacts, leading to significantly more stable downstream Bayesian CRB estimates compared with the original SSM 
objective. Furthermore, the data-informed prior was 
shown to play a significant role in the structure of the resulting bound, emphasizing the importance of accurately 
capturing the prior distribution when using the bound for downstream tasks such as optimal experimental design 
and estimator benchmarking. 

This work represents a promising step toward making the Bayesian CRB a practical tool for the design and analysis of high-dimensional imaging systems. However, several 
limitations remain 
to be addressed to realize this potential. First, the proposed approach was evaluated 
using a highly 
stylized model of the underlying imaging physics. In particular, the nonlinear PDE-based models typically used to model 
light propagation in qPACT were sidestepped, leading to a 
simple linear-Gaussian likelihood function. We also used a two-dimensional model of the imaging system, whereas most 
qPACT systems used in practice are inherently three-dimensional. Moreover, we neglected 
the 
effects of unknown tissue parameters that are not of direct clinical interest, 
such as optical scattering, speed-of-sound heterogeneity, and acoustic attenuation. These parameters are known 
to play a significant role in qPACT image formation \cite{cox2009challenges}. New techniques will likely be required to address the additional computational challenges associated with these complications. Second, rigorous comparisons of the bound estimates obtained from 
the proposed framework with the error maps of estimators for 
the corresponding inverse problem are needed. Examining the relationship 
between the bound and the error maps would help assess 
the validity of the bound estimate as a meaningful estimator benchmark and as a guide for experimental design. 

The Bochner-space SSM objective introduced here also warrants further investigation as a stand-alone technique for score learning. In 
this work, we avoided directly formulating the Bochner-space SSM objective in the infinite-dimensional function space setting that motivated its development. 
However, there has been substantial recent work on the development of infinite-dimensional SM 
approaches in the context of denoising diffusion models \cite{lim2023score, pidstrigach2023infinite}. A formal comparison between Bochner-space SSM and Sobolev-norm denoising SM 
could provide key insights into the trade-offs between computational efficiency and the preservation of function-space structure.

\end{document}